\documentclass[twoside,11pt]{article}

\usepackage{jmlr2e}
\usepackage[cmex10]{amsmath}
\usepackage{amsfonts, bm}
\usepackage{array}
\usepackage{color}
\usepackage{algorithm, algorithmic}
\usepackage[active]{srcltx}
\usepackage{subfig}
\usepackage{hyperref}
\usepackage{rotating}

\def\norm#1{\|#1\|}

\def\R{{\mathbb R}}

\def\be{\begin{eqnarray*}}
\def\ee{\end{eqnarray*}}
\def\beq{\begin{equation}}
\def\eeq{\end{equation}}

\def\2q{\quad\quad}

\def\R{{\mathbb R}}

\def\:{{\,:\,}}

\def\norm#1{{\left\|\,#1\,\right\|}}
\def\abs#1{{\left|\,#1\,\right|}}

\def\norm #1{\|#1\|}
\def\abs #1{|#1|}

\def\define{:=}
\def\inprod#1#2{\langle #1,\,#2\rangle}

\def\itb{\begin{itemize}}
\def\ite{\end{itemize}}
\def\bit{\begin{itemize}}
\def\eit{\end{itemize}}
\def\dom{\hbox{dom}}

\begin{document}
 
\title{The Bregman-Tweedie Classification Model}

\author{\name Hyenkyun Woo \email hyenkyun@koreatech.ac.kr, hyenkyun@gmail.com \\
       }

\editor{}

\maketitle

\begin{abstract}
This work proposes the Bregman-Tweedie classification model and analyzes the domain structure of the extended exponential function, an extension of the classic generalized exponential function with additional scaling parameter, and related high-level mathematical structures, such as the Bregman-Tweedie loss function and the Bregman-Tweedie divergence. The base function of this divergence is the convex function of Legendre type induced from the extended exponential function. The Bregman-Tweedie loss function of the proposed classification model is the regular Legendre transformation of the Bregman-Tweedie divergence. This loss function is a polynomial parameterized function between unhinge loss and the logistic loss function. Actually, we have two sub-models of the Bregman-Tweedie classification model; H-Bregman with hinge-like loss function and L-Bregman with logistic-like loss function.  Although the proposed classification model is nonconvex and unbounded, empirically, we have observed that the H-Bregman and L-Bregman outperform, in terms of the Friedman ranking, logistic regression and SVM and show reasonable performance in terms of the classification accuracy in the category of the binary linear classification problem. 
\end{abstract}


\section{Introduction\label{sec1}}
The exponential function is an essential and fundamental function while explaining various data dependent problems appearing in machine learning, such as regression, classification and clustering. However, the observed data is not always well explained through the classic exponential function and thus the generalization of this function is required. The well-known generalization is the $\alpha$-exponential function~\citep{tsallis09}, defined as $\overline{\exp_{\alpha}}(x) = (1 + (1-\alpha)x)^{\frac{1}{1-\alpha}}$ (or $\max(0,1+(1-\alpha)x)^{\frac{1}{1-\alpha}}$). The corresponding generalized $\alpha$-logarithmic function is defined as $\overline{\ln_{\alpha}}(y) = \frac{y^{1-\alpha}-1}{1-\alpha}$. See \citep{amari16}, for more details on the generalized elementary function and its various applications. Note that mathematical structures, like domain and inverse relation, of these generalized elementary functions are not well studied. For instance, if $\alpha=1/3$, $\overline{\exp_{\alpha}}\circ\overline{\ln_{\alpha}}(-8) = 8$ and if $\alpha=2/3$, $\overline{\exp_{\alpha}}\circ\overline{\ln_{\alpha}}(-8) = -8$. Therefore, the characterization of domains satisfying the inverse relation (and the high-level structure like convex function of Legendre type) is highly demanded. 

Recently, \citep{ding10, ding11} have proposed $\alpha$-logistic regression model with the $\alpha$-exponential families. This is based on the generalized $\alpha$-exponential function to obtain robustness on the label noise in classification problem. However, when we try to directly generalize the logistic regression with the generalized logarithmic function and the generalized exponential function, because of lack of inverse relation and ambiguity of the domains, it is unclear how to generalize the classic logistic loss function to include the hinge-like loss function.

Interestingly, two decades ago, \citep{lafferty99} had suggested Bregman-beta divergence, which is induced from the beta-divergence, to study the generalized boosting model satisfying the additive structure of the adaboosting. As noticed in \citep{woo17}, the additive structure of the generalized boosting model is well-defined when the base function of the Bregman-beta divergence is a convex function of Legendre type~\citep{roc70, bauschke97}. Inspired from~\citep{lafferty99}, in this article, we study an extended exponential function (i.e., the generalized exponential function with additional scaling parameter) and fully characterize its domain structure. Also, we explore the high-level mathematical structures on the extended exponential function, for instance, the convex function of Legendre type and the Bregman-Tweedie divergence. The base function of the divergence is the convex function of Legendre type induced from the extended exponential function. Moreover, we propose the Bregman-Tweedie classification model, the loss function of it is the regular Legendre transformation of the Bregman-Tweedie divergence which is a dual formulation of the Bregman-beta divergence. For more details on these divergences, see \citep{woo17,woo18}.

The Bregman-Tweedie loss function is a polynomial parameterized loss function between the unhinge loss and the logistic loss function. Depending on the choice of the parameters, we have two sub-models of the Bregman-Tweedie classification model; H-Bregman with hinge-like loss function ($c_{\alpha}=1$) and L-Bregman with logistic-like loss function ($c=1$). Note that, very recently, we have suggested a general framework of the convex classification model, Logitron~\citep{woo19}. This is the Perceptron-augmented extended logistic regression model. As opposed to the Logitron loss function, the proposed Bregman-Tweedie  loss function is non-convex and unbounded below. However, by virtue of the projection-based optimization~\citep{mark19} and rescaling of the data space, the proposed Bregman-Tweedie classification model shows reasonable performance in terms of classification accuracy and outperforms logistic regression and SVM, in terms of Friedman ranking, when $\alpha \approx 1$. We have used the state-of-the-art linear classification benchmark package; logistic regression, SVM, and L2SVM in LIBLINEAR~\citep{fan08} and the several dozens of UCI benchmark dataset in the category of binary classification. Last but not least, the Bregman-Tweedie divergence is an essential function while characterizing the structure of the Tweedie exponential dispersion model~\citep{jorgensen97,bar-lev86}. See \citep{woo18} for more details on the moment-limited statistical distribution including Tweedie exponential dispersion model. 

Before we introduce the extended exponential function and the Bregman-Tweedie classification model, several useful notations, frequently used in this article, are introduced. $\R_+ = \{ x \in \R \;|\; x \ge 0 \}$, $\R_{++} = \{ x \in \R \;|\; x > 0 \}$, $\R_{-} = \{ x \in \R \;|\; x \le 0 \}$, and $\R_{--} = \{ x \in \R \;|\; x < 0 \}$. ${\mathbb Z}$ is a set of integer and ${\mathbb Z}_+ = \{0,1,2,... \}$. The following categorization~\citep{woo17} of real line $\R$ is indispensable while clarifying the domain and the range of the extended exponential function and the high-level mathematical structures, including the Bregman-Tweedie classification model.
\begin{equation}\label{Rclass} 
\begin{array}{l}
\R_e = \{ 2k/(2l+1) \;|\; k,l \in \mathbb{Z} \}\\
\R_o = \{ (2k+1)/(2l+1) \;|\; k,l \in \mathbb{Z} \}\\
\R_{x} = \R \setminus (\R_o \cup \R_e)
\end{array}
\end{equation}
where $\R_x$ can be divided into two sub-categories:
$\R_{xe} = \left\{ (2k+1)/2l \;|\; k \in \mathbb{Z},\; l \in \mathbb{Z}\setminus\{0\} \right\}$,  
$\R_{xx} = \R_x \setminus \R_{xe}$. The inverse relations are also useful. $\R_o^{-1} =  \R_o$,
$\R_{xe} = (\R_e\setminus\{ 0 \})^{-1 }$, 
$\R_{xx} = \R_{xx}^{-1}$ 
where, for instance, $(\R_e\setminus\{ 0 \})^{-1} =  \R_{xe}$ means that, for all $a \in (\R_e \setminus \{ 0 \})^{-1}$, we have $a^{-1} \in \R_{xe}$ and vice versa. Note that $bd(\Omega)$ is a boundary of $\Omega$ and $int(\Omega)$ is an interior of $\Omega$. $\inprod{a}{b} = \sum_{i=1}^n a_ib_i$ where  $a,b \in \R^n$. In addition, we assume that the domain of a function is a {\it convex set}, irrespective of the convexity of the function. 

We briefly overview the organization of this work. In Section \ref{sec2}, we introduce the extended exponential function and the Bregman-Tweedie divergence. In Section \ref{sec3}, the Bregman-Tweedie loss function and the Bregman-Tweedie classification model are studied based on the Bregman-Tweedie divergence. The numerical experiments of the Bregman-Tweedie classification model in binary classification problem is presented in Section \ref{sec4}. The conclusions are given in Section \ref{sec6}. 

\section{The extended exponential function and the Bregman-Tweedie divergence~\label{sec2}}
This Section presents the extended exponential function~\citep{woo19} and the high-level mathematical structures based on it. That is, the convex function of Legendre type (the indefinite integral of the extended exponential function with the reduced domain) and the Bregman-Tweedie divergence.  Also, the domain and the range of the various extended exponential related functions are carefully analyzed with the category of the real line $\R$ in \eqref{Rclass}~\citep{woo17}.

\begin{figure*}[t]
\centering
\includegraphics[width=5.5in]{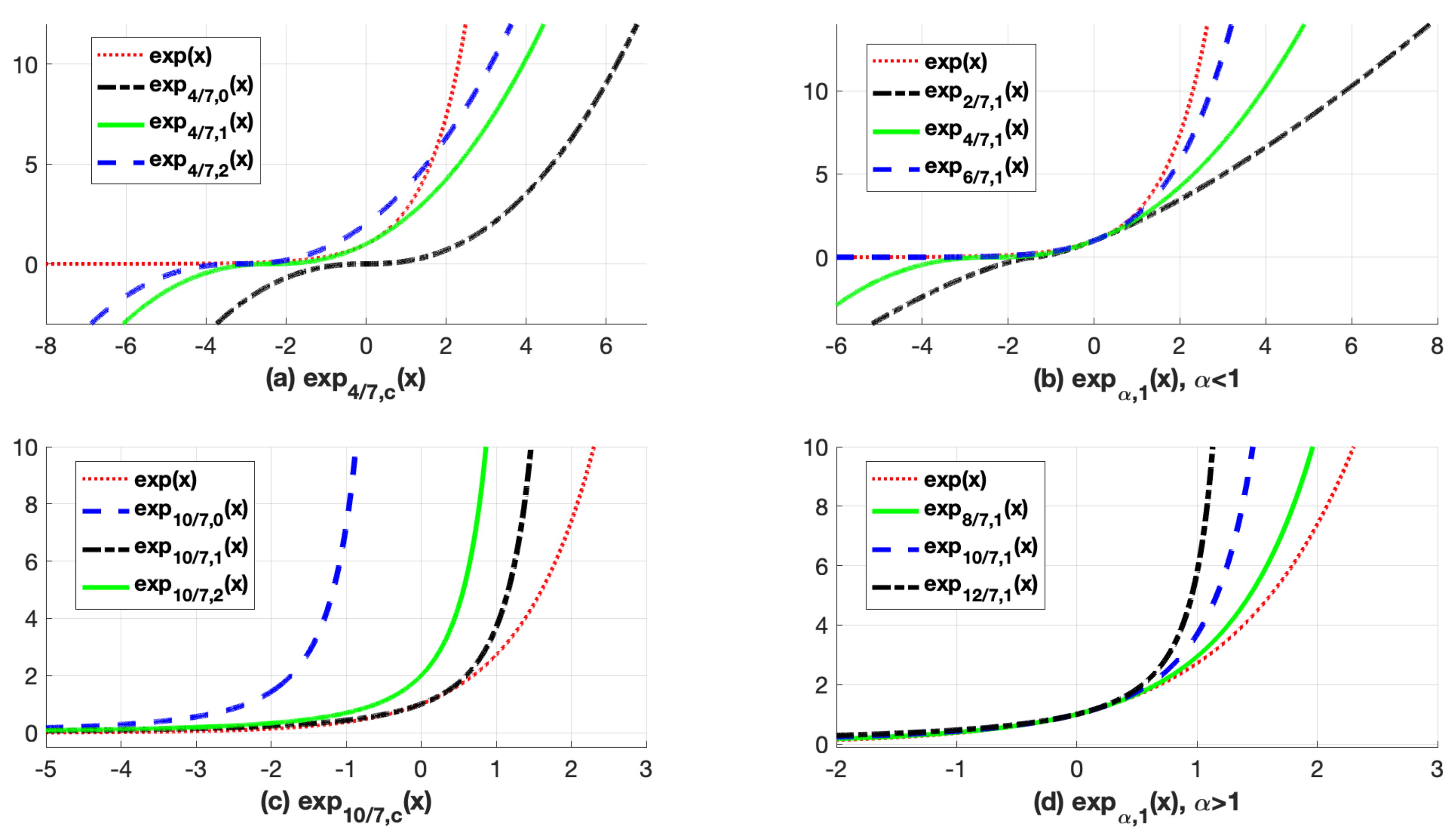} 
\caption{The graphs of the extended exponential functions with various different choices of $\alpha$ and $c$. (a) $\exp_{4/7,c}$ with $c=0,1,2$. (b) $\exp_{\alpha,1}$ with $\alpha = 2/7,4/7,6/7$. (c) $\exp_{10/7,c}$ with $c=0,1,2$. (d) $\exp_{\alpha,1}$ with $\alpha=8/7,10/7,12/7$. When $\alpha \in \{ (0,1) \cap \R_e \} \cup \{ 0, 1\}$, we have $\dom(\exp_{\alpha,c}) = \R$.}
\label{fig:img1}
\end{figure*}

Let us start with the definition of the extended exponential function~\citep{woo19}. Inspired from \citep{lafferty99}, we reformulate it with the generalized exponential function $\overline{\exp}_{\alpha}(x)$:
\begin{equation}\label{expraw}
\boxed{
\exp_{\alpha,c}(x) = \overline{\exp_{\alpha}}(x + \overline{\ln_{\alpha}}(c)) = (c^{1-\alpha} + (1-\alpha)x)^{\frac{1}{1-\alpha}}
}
\end{equation}
where $\exp_{1,c}(x) = c\exp(x)$ and $\dom(\exp_{\alpha,c}) = \{ x \in \R \;|\; \exp_{\alpha,c}(x) \in \R \}$. As observed in \citep{bar-lev86,woo17}, it is more convenient to use an equivalence class for the extended exponential function~\eqref{expraw}. 
\begin{definition}\label{def:exp}
Let $\alpha \in \R$ and $x \in \dom(\exp_{\alpha})$. Then the extended exponential function in \eqref{expraw} is simplified as
\begin{equation}\label{exexp}
\exp_{\alpha}(x) \define
\left\{\begin{array}{l} 
\exp([x]), \qquad\qquad\quad\;\; \hbox{ if } \alpha=1\\
((1-\alpha)[x])^{1/(1-\alpha)} , \quad \hbox{ otherwise }
\end{array}\right.
\end{equation}
where  $\dom(\exp_{\alpha}) = \{ x \in \R \;|\; \exp_{\alpha}(x) \in \R \}$ is in Table \ref{tableE} and $c_{\alpha} = \frac{c^{1-\alpha}}{\alpha-1} \in \R$. If $\alpha \not= 1$, $[x] = x - c_{\alpha}$ and if $\alpha = 1$, $[x] = x + \ln(c)$. Note that, when $\alpha>1$, $sign(c) = sign(\exp_{\alpha}(x))$. 
\end{definition}
For simplicity, in the following of the article, we use $x$, instead of the equivalence class $[x]$, unless otherwise stated. Though we use an equivalence class~\eqref{exexp} for the extended exponential function~\eqref{expraw}, the role of $c$ (i.e., $c_{\alpha}$) is so important in machine learning. In \citep{woo19}, we show that the higher-order hinge loss function, which frequently used as loss functions in machine learning, is a special case of the Perceptron-augmented extended exponential function, i.e.,
$
\hbox{Exp}_{\alpha,c}(x) = \left(\max\left(0,c^{1-\alpha}-(1-\alpha)x\right)\right)^{\frac{1}{1-\alpha}} 
$
where $c>0$ and $0 \le \alpha<1$. For instance, we get the famous hinge-loss function $\hbox{Exp}_{0,1}(x) = \max(0,1-x)$ and the squared hinge-loss function (or L2SVM) $\hbox{Exp}_{1/2,1/4}(x) = 4^{-1}(\max(0,1-x))^2$~\citep{fan08}. Note that third order hinge-loss function is used as an activation function of the deep neural network~\citep{janocha17}. However, if the classic generalized exponential function~\citep{ding10} (i.e. $c=1$) is used then the margin depends only on $\alpha$ and thus we could not control it. The details explanations are given in Section \ref{sec3}.

Note that additional conditions on  $dom(\exp_{\alpha})$ are required for high-level mathematical structures based on the extended exponential function $\exp_{\alpha}$. For instance, the convex function of Legendre type and the inverse relation with the extended logarithmic function.  Now, let us consider the extended logarithmic function~\citep{woo17}, which is formulated with the classic generalized logarithmic function~\citep{tsallis09,amari16}:
\begin{equation}\label{lograw}
\boxed{
\ln_{\alpha,c}(x) = \overline{\ln_{\alpha}}(x) - \overline{\ln_{\alpha}}(c) = c_{\alpha} - x_{\alpha}
}
\end{equation}
where $\ln_{1,c}(x) = \ln(x) - \ln(c)$.
This function is also reformulated with the equivalence class. That is, $\ln_{\alpha}(x) =  [\ln_{\alpha,c}(x)] = \ln_{\alpha,c}(x) - c_{\alpha}$. The details are following.
\begin{definition}\label{def:log}
Let $\alpha \in \R$ and $x \in \dom(\ln_{\alpha})$ then
\begin{equation}\label{exlog}
\ln_{\alpha}(x) \define
\left\{\begin{array}{l} 
\ln(x), \qquad\quad\;\; \hbox{ if } \alpha=1\\
\frac{1}{1-\alpha}x^{1-\alpha} , \qquad \hbox{ otherwise }
\end{array}\right.
\end{equation}
where $\dom(\ln_{\alpha})$ is in Table \ref{tableL}.
\end{definition}
As observed in Table \ref{tableE} and Table \ref{tableL}, $\exp_{\alpha}$ and $\ln_{\alpha}$ do not have the inverse relation. The partial inverse relation between them is summarized in Lemma \ref{lemma4} (in Appendix). The following Lemma presents the reduced domains of them for the inverse relation between $\exp_{\alpha}$ and $\ln_{\alpha}$.

\begin{table*}[t]
\centerline{
{\scriptsize
\begin{tabular}{c|c|ccc|ccc}
\hline\hline
   & $\alpha = 1$ &  & $\alpha < 1$ & & & $\alpha > 1$ & \\ \cline{3-8} 
                 &   & $1-\alpha \in \R_{xe}$ &  $1-\alpha \in \R_o$ &  $1-\alpha  \in \R_{xx} \cup \R_e$   & $1-\alpha \in \R_{xe}$ &  $1-\alpha \in \R_o$ &  $1-\alpha \in \R_{xx}\cup \R_e$ \\ \hline 
 $\dom (\exp_{\alpha})$ & $\R$ & $\R$ & $\R$ & $\R_+$ & $\R_{++}$ / $\R_{--}$ & $\R_{++}$ / $\R_{--}$ & $\R_{--}$ \\ \hline 
 $\hbox{ran}(\exp_{\alpha})$   & $\R_{++}$  & $\R_+$ & $\R$ & $\R_+$ & $\R_{++}$ & $\R_{--}$ / $\R_{++}$ & $\R_{++}$ \\ \hline\hline
\end{tabular}
}}
\caption{The domain and the range of the extended exponential function $\exp_{\alpha}$ in \eqref{exexp}.}\label{tableE}
\end{table*}

\begin{table*}[t]
\centerline{
{\scriptsize
\begin{tabular}{c|c|ccc|ccc}
\hline\hline
   & $\alpha = 1$ &  & $\alpha < 1$ & & & $\alpha > 1$ & \\ \cline{3-8} 
                 &   & $1-\alpha \in \R_e$ &  $1-\alpha  \in \R_o$ &  $1-\alpha \in \R_x$   & $1-\alpha \in \R_e$ &  $1-\alpha \in \R_o$ &  $1-\alpha \in \R_x$ \\ \hline 
 $\dom (\ln_{\alpha})$ & $\R_{++}$ & $\R$ & $\R$ & $\R_+$ & $\R_{++}$ / $\R_{--}$ & $\R_{++}$ / $\R_{--}$ & $\R_{++}$ \\ \hline 
 $\hbox{ran}(\ln_{\alpha})$   & $\R$  & $\R_+$ & $\R$ & $\R_+$ & $\R_{--}$ & $\R_{--}$ / $\R_{++}$ & $\R_{--} $ \\ \hline\hline
\end{tabular}
}}
\caption{The domain and the range of the extended logarithmic function $\ln_{\alpha}$ in \eqref{exlog}.}\label{tableL}
\end{table*}

\begin{table*}[h]
\centerline{\scriptsize
\begin{tabular}{c|c|ccc|ccc}
\hline\hline
   & $\alpha = 1$ &  & $\alpha < 1$ & & & $\alpha > 1$ & \\ \cline{3-8} 
                 &   &  &  $\alpha \in \R_e$ &  $\alpha \in \R \setminus \R_e$   & &  $\alpha \in \R_e$ &  $\alpha \in \R \setminus \R_e$ \\ \hline 
 $\dom (\exp_{\alpha})$   & $\R$  & & $\R$ & $\R_+$ &  & $\R_{--}$ / $\R_{++}$ & $\R_{--} $ \\ \hline
 $\dom (\ln_{\alpha})$ & $\R_{++}$ &  & $\R$ & $\R_+$ & & $\R_{++}$ /  $\R_{--}$ & $\R_{++}$ \\ \hline\hline 
\end{tabular}}
\caption{The reduced domains of the extended exponential function $\exp_{\alpha}$ and the extended logarithmic function $\ln_{\alpha}$ for the bijection  $\exp_{\alpha}= \ln_{\alpha}^{-1} : \dom (\exp_{\alpha}) \rightarrow \dom (\ln_{\alpha})$.
}\label{table3}
\end{table*}

\begin{lemma}\label{lemmaD}
With the reduced domains of $\exp_{\alpha}$ and $\ln_{\alpha}$ in Table \ref{table3}, we have a bijective map: 
\begin{equation}\label{xexpwithD}
\exp_{\alpha}: \dom (\exp_{\alpha}) \rightarrow \dom (\ln_{\alpha})
\end{equation}
and $\ln_{\alpha} = \exp_{\alpha}^{-1}$.
\end{lemma}
\begin{proof}
When $\alpha=1$, we have ordinary log and exp functions. Now, we assume that $\alpha \not= 1$. Let $\alpha \in \R_e$ ($1-\alpha \in \R_o$), then $\exp_{\alpha}^{-1}(x) =  \ln_{\alpha}(x) = -x_{\alpha}$ is a monotonic function. In addition, when $\alpha<1$, it is bijective between $\dom(\exp_{\alpha})$ and $\dom(\ln_{\alpha})$. Since the domain of a function is convex, when $\alpha>1$, we have two possible choices of domain $\R_{++}/\R_{--}$ and the corresponding bijective map on the domain.
	
Now, let us consider $\alpha \in \R \setminus \R_{e}$. Based on the analysis in Lemma \ref{lemma4} (Appendix), we reduce the domain of $\ln_{\alpha}$ in Table \ref{tableL} to $\dom(\ln_{\alpha}) \cap \R_{+}$. 
That is, we set $\dom(\ln_{\alpha}) \define \dom(\ln_{\alpha}) \cap \R_+$. In case of the extended exponential function, it is a little bit complicated. When $\alpha<1$, we set $\dom(\exp_{\alpha}) \define \dom(\exp_{\alpha}) \cap \R_{+}$ and, when $\alpha>1$, we set $\dom(\exp_{\alpha}) \define \dom(\exp_{\alpha}) \cap \R_{--}$.  
Then we have 
$\hbox{ran}(\ln_{\alpha}) = \dom(\exp_{\alpha}) = \left\{\begin{array}{l} \R_+  \;\;\quad\hbox{ if } \alpha<1 \\ \R_{--} \quad\hbox{ if } \alpha>1 \end{array}\right.$ 
and 
$\hbox{ran}(\exp_{\alpha}) = \dom(\ln_{\alpha}) = \left\{\begin{array}{l} \R_+  \;\;\quad\hbox{ if } \alpha<1 \\ \R_{++} \quad\hbox{ if } \alpha>1 \end{array}\right.$ 
Due to the monotonicity of $x^{a}$ for all $a \in \R \setminus \{ 0 \}$ on the reduced domain in Table \ref{table3}, one-to-one condition is satisfied. 
\end{proof}
By virtue of the reduced domains in Table \ref{table3}, it is easy to build up sophisticated mathematical objects, like convex function of Legendre type. In Lemma \ref{psiTh} (Appendix), we briefly describe the domain of $\Psi(x) = \int_d^x \exp_{\alpha}(\xi)d\xi$ only with the domain condition~\citep{roc70}. That is, for a convex function $f$, we have 
$int(\dom  f) \subseteq \dom \partial f \subseteq \dom  f$, where $\partial f$ is a subgradient of $f$. Based on this, we can characterize the domain of $\Psi(x)$ satisfying the additional conditions of the convex function of Legendre type. Note that $\Psi(x)$ is a useful function while analyzing the structure of the Bregman-Tweedie classification model presented in Section \ref{sec3} and the moment-limited Tweedie exponential dispersion model~\citep{woo18}. 

Let us start with the definition of the convex function of Legendre type~\citep{bauschke97,roc70}.
\begin{definition}\label{legendredef}
Let $f: \dom f \rightarrow \R$ be lower semicontinuous, convex, proper function on $\dom f \subseteq \R$. Then $f$ is a convex function of Legendre type, if the following conditions are satisfied.
\begin{itemize}
\item $int(\dom f)\not= \emptyset$ and $f$ is strictly convex and differentiable on $int(\dom f)$
\item ({\it steepness}) $\forall x \in bd(\dom f)$ and $\forall y \in int(\dom f),$
\begin{equation}\label{steep}
\lim_{t \downarrow 0} \inprod{f'(x + t(y-x))}{y-x} = -\infty
\end{equation}
\end{itemize}
\end{definition}
Here, \eqref{steep} is known as the steepness condition in statistics~\citep{brown86,barndorff14}.  Now, we present $\Psi$, the indefinite integral of the extended exponential function, satisfying the conditions of the convex function of Legendre type in Definition \ref{legendredef}.
\begin{theorem}\label{corPsi}
Let  $x \in \dom\Psi$ and $\exp_{\alpha}$ be an extended exponential function in \eqref{xexpwithD}. Then,
\begin{equation}\label{basefnD}
\Psi(x) = \int_d^{x} \exp_{\alpha}(\xi)d\xi = 
\left\{\begin{array}{l} 
\exp(x) \qquad\qquad\quad\;\; \hbox{ if } \alpha = 1 \\
-\ln(-x) \qquad\quad\quad\;\; \hbox{ if } \alpha = 2 \\
\frac{1}{2-\alpha}[(1-\alpha)x]^{\frac{2-\alpha}{1-\alpha}} \quad\;\; \hbox{otherwise }
\end{array}\right.
\end{equation}
is the convex function of Legendre type on $\dom  \Psi$:
\begin{equation}\label{conditionLegendreD}
\left\{
\begin{array}{l}
\hbox{I. entire region:}\\
\qquad \alpha<1,\; \alpha \in \R_e  \hskip 1cm\hbox{ and } \hskip 0.5cm \dom\Psi = \R,\\
\qquad \alpha=1,\; \hskip 2.2cm\hbox{ and } \hskip 0.5cm \dom\Psi = \R,\\
\hbox{II. positive region:}\\
\qquad 1 < \alpha < 2, \alpha \in \R_e \hskip 0.45cm \hbox{ and }  \hskip 0.5cm  \dom\Psi = \R_{++},\\
\qquad 2 < \alpha, \qquad \alpha \in \R_e \hskip 0.4cm \hbox{ and }  \hskip 0.5cm  \dom\Psi = \R_{+},\\
\hbox{III. negative region:}\\
\qquad 1 \le \alpha < 2,  \hskip 1.65cm\hbox{ and }  \hskip 0.5cm \dom\Psi = \R_{--},\\
\qquad 2<\alpha,  \hskip 2.4cm \hbox{ and }  \hskip 0.55cm \dom\Psi = \R_{-}. \\
\end{array}
\right.
\end{equation}
Here, we drop all constant terms.
\end{theorem}
\begin{proof}
It is trivial to show that \eqref{basefnD} is the convex function of Legendre type when $\alpha \in \{ 1, 2 \}$. Hence, we assume that $\alpha \in \R \setminus \{1,2\}.$ It is easy to check differentiability of $\Psi$ on $int(\dom\Psi)$ and thus we only need to check steepness condition and strict convexity of $\Psi$ on $int(\dom \Psi)$. As observed in Table \ref{table3}, we have $\partial\Psi(0) = \exp_{\alpha}(0) = 0$, only when $\alpha<1$ and $\alpha \in \R \setminus \R_e.$ Thus, steepness condition is not satisfied in this case. 

Now, we will check strict convexity of $\Psi$ on $int(\dom \Psi)$. The second derivative of $\Psi$ is given as
\begin{equation}\label{hessianexp}
\Psi''(x) = ((1-\alpha)x)^{\frac{\alpha}{1-\alpha}}
\end{equation}
where $\alpha \not\in \{1,2\}$.
\begin{itemize}
\item $\alpha<1$: 
\begin{itemize}
\item $\alpha \in \R_e$: $\frac{\alpha}{1-\alpha} \in \R_e$. Thus  we get $\Psi''(x) >  0$ for all $x \in \R \setminus \{  0 \}$. We only need to check strict convexity at zero. In fact,
$$
\Psi'(x) = \exp_{\alpha}(x) = [(1-\alpha)x]^{1/(1-\alpha)} =  c_1 x^{\frac{1}{1-\alpha}},
$$
where $c_1 = (1-\alpha)^{\frac{1}{1-\alpha}}>0$ and $\Psi'(-x) = -\Psi'(x).$ Thus, $\Psi'(x)$ is monotonically increasing at zero. Hence, $\Psi$ is strictly convex on its entire domain $\R.$
\item $\alpha \in \R \setminus \R_e$: $\frac{\alpha}{1-\alpha} \in \R \setminus \R_e.$ Hence, we have $\Psi''(x)>0,$ for all $x \in int(\dom \Psi) = \R_{++}.$ Therefore, it satisfies strict convexity on $int(\dom \Psi)$. 
\end{itemize}
\item $1<\alpha<2$: 
The domain of $\Psi$ is open, i.e., $\dom \Psi = \R_{--}$ (or $\R_{++}$). Therefore, we only need to check $\Psi''(x)> 0$ on its domain.
\begin{itemize}
\item $\alpha \in \R_e$: $\frac{\alpha}{1-\alpha} \in \R_e$. Thus, we get $\Psi''(x) > 0$ when $x \in \R_{++}$ (or $\R_{--}$). Therefore, $\Psi$ is strictly convex on $int(\dom \Psi)$ where $\dom \Psi = \R_{++}$ (or $\R_{--}$).
\item $\alpha \in \R \setminus \R_e$: $\frac{\alpha}{1-\alpha} \in \R \setminus \R_e$. When $x \in \R_{--}$, we get $\Psi''(x)> 0$. Therefore, $\Psi$ is strictly convex on $int(\dom \Psi)$ where $\dom \Psi = \R_{--}$. 
\end{itemize}
\item $\alpha>2:$
\begin{itemize}
\item $\alpha \in \R_e$: $\frac{\alpha}{1-\alpha} \in \R_e$. When $x \in \R_{++}$ (or $\R_{--}$), we get $\Psi''(x) > 0.$  Therefore, $\Psi$ is strictly convex on its interior of domain, i.e., $int(\dom \Psi) = \R_{++}$ (or $\R_{--}$).  
\item $\alpha \in \R \setminus \R_{e}$: $\frac{\alpha}{1-\alpha} \in \R \setminus \R_e$. When $x \in \R_{--}$, we get $\Psi''(x) > 0$ and therefore $\Psi$ is strictly convex on $int(\dom \Psi)$.
\end{itemize}
\end{itemize}
We conclude that $\Psi$ is the convex function of Legendre type on its domain defined in Lemma \ref{psiTh} (Appendix) except $\alpha < 1$ and $\alpha \in \R\setminus\R_e.$ 
\end{proof}

\begin{figure*}[t]
\centering
\includegraphics[width=5.5in]{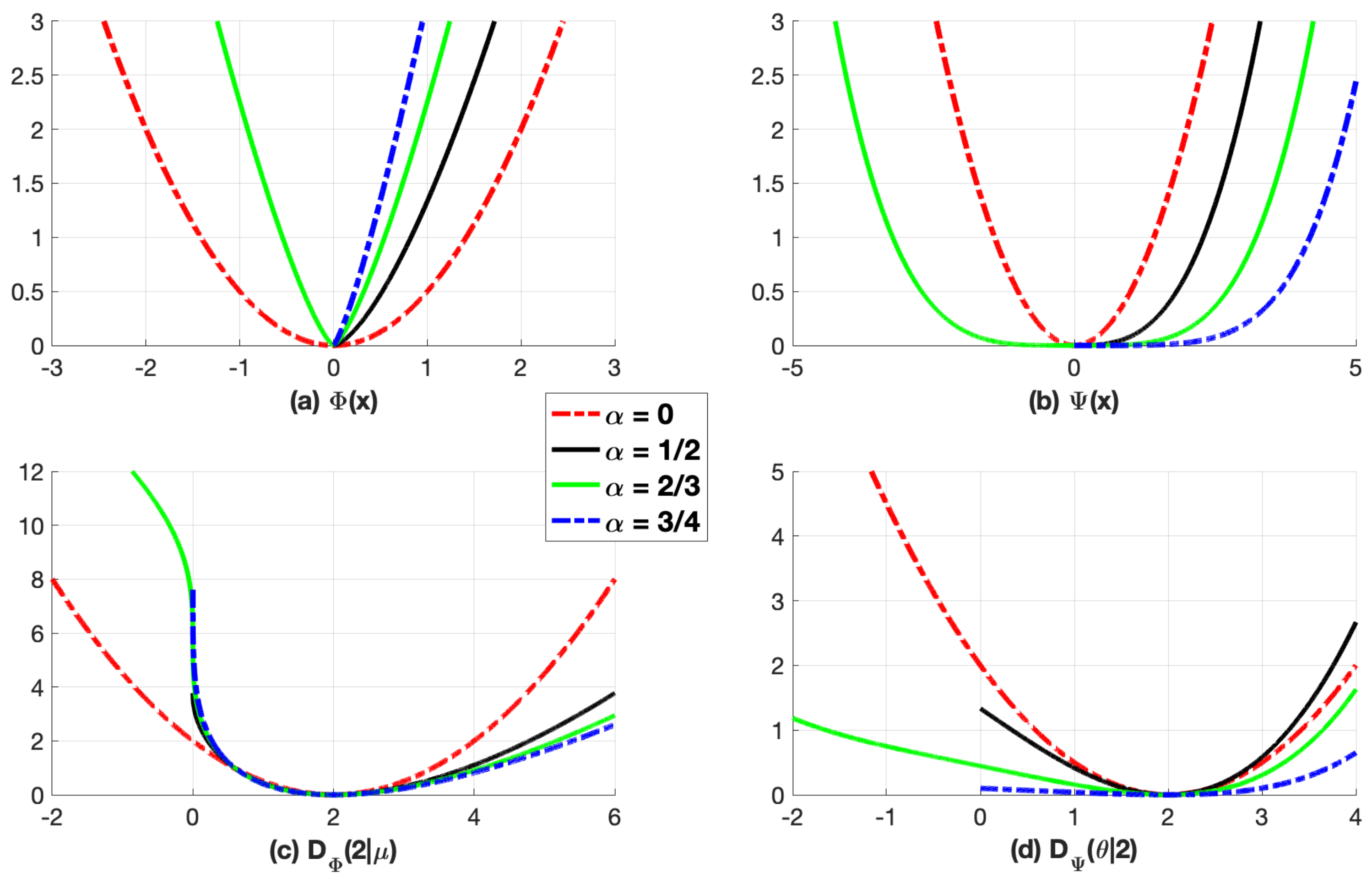} 
\caption{Graphs of the convex function of Legendre type; (a) $\Phi(x) = \frac{1}{(2-\alpha)(1-\alpha)}x^{2-\alpha}$ and (b) $\Psi(x) = \frac{1}{2-\alpha} [(1-\alpha)x]^{\frac{2-\alpha}{1-\alpha}}$, and the Bregman divergence associated with them; (c) Bregman-beta divergence $D_{\Phi}(2|\mu)$ and (d) Bregman-Tweedie divergence $D_{\Psi}(\theta|2)$. Here $\alpha = 0, 1/2, 2/3, 3/4$.}
\label{fig:img2}
\end{figure*}

Interestingly, as observed in \citep{woo19}, $\Psi$ is a cumulant function of the Tweedie exponential dispersion model and satisfy the conditions of the convex function of Legendre type in Definition \ref{legendredef} at the same time.
\begin{equation}\label{pX}
p_{\Psi}(b;\theta,\sigma^2) = \exp\left(\frac{\inprod{b}{\theta}-\Psi(\theta)}{\sigma^2}\right)p_0(b,\sigma^2)
\end{equation}
where $b \in {\cal B}$ is a random variable, $\sigma^2>0$ is a dispersion parameter, and $p_0(b,\sigma^2)$ is a base measure satisfying $\int_{{\cal B}} p_{\Psi}(b;\theta,\sigma^2)\nu(db) = 1.$ Note that \eqref{pX} can be reformulated with the Bregman-divergence associated with $\Psi$. See also \citep{banerjee05} for the equivalence between regular exponential families and the regular Bregman-divergence. Now, we define the {\it Bregman-Tweedie divergence} (i.e., the Bregman-divergence associated with $\Psi$) as  
\begin{equation}\label{tweedieBregman}
D_{\Psi}(x|y) = \Psi(x) - \Psi(y) - \inprod{\Psi'(y)}{x-y}
\end{equation}
where $(x,y) \in \dom\Psi \times int(\dom\Psi)$. Here, $\dom\Psi$ is in \eqref{conditionLegendreD}. Though the Bregman-Tweedie divergence~\eqref{tweedieBregman} is useful in characterizing Tweedie distribution~\citep{woo19}, it is also helpful in understanding the Bregman-Tweedie classification model in Section \ref{sec3}. However, $\Phi(x)= \int_d^x \ln_{\alpha}(t)dt$ was studied in \citep{woo17} for the characterization of the $\beta$-divergence within the Bregman divergence framework. Hence, we have the {\it Bregman-beta divergence} (i.e., Bregman-divergence associated with $\Phi$) defined as
\begin{equation}\label{betaBregman}
D_{\Phi}(x|y) = \Phi(x) - \Phi(y) - \inprod{\Phi'(y)}{x-y}
\end{equation}
where $\Phi$ is in the following theorem.
\begin{theorem}\label{legendre_beta}
Let $x \in dom\Phi$ and 
\begin{eqnarray}\label{basefn}
\Phi(x) &=& \int_d^{x} \ln_{\alpha}(t) dt
= \left\{\begin{array}{l} 
 -\log x, \qquad\quad\quad \;\;\hbox{ if } \alpha=2,\\ 
x\log x - x, \qquad\quad \hbox{ if } \alpha =1,\\
\frac{1}{(2-\alpha)(1-\alpha)}x^{2-\alpha}, \;\;\quad \hbox{ if } \alpha \not= 1,2,
\end{array}\right.
\end{eqnarray}
where  $\ln_{\alpha}(t)$ is the extended logarithmic function in Definition \ref{def:log}. Then, $\Phi$ in \eqref{basefn} is the convex function of Legendre type on $\dom\Phi$ given below: 
\begin{equation}\label{conditionLegendre}
\left\{
\begin{array}{l}
\hbox{I. entire region:}\\
\qquad \alpha<1,\; \alpha \in \R_e  \hskip 1.22cm\hbox{ and } \hskip 0.5cm \dom\Phi = \R,\\
\hbox{II. positive region:}\\
\qquad  1 \le  \alpha < 2 \hskip 1.9cm \hbox{ and }  \hskip 0.5cm \dom\Phi  = \R_{+},\\
\qquad 2 \le \alpha  \hskip 2.6cm \hbox{ and }  \hskip 0.5cm  \dom\Phi = \R_{++},\\
\hbox{III. negative region:}\\
\qquad 1< \alpha < 2, \; \alpha \in \R_e  \hskip 0.53cm\hbox{ and }  \hskip 0.2cm \dom\Phi = \R_{-},\\
\qquad 2 < \alpha , \;\qquad \alpha \in \R_e \hskip 0.43cm \hbox{ and }  \hskip 0.3cm \dom\Phi = \R_{--}. \\
\end{array}
\right.
\end{equation}
For simplicity, we drop all constants in $\Phi(x)$.
\end{theorem}
As observed in \citep{woo17}, due to the invariance properties with respect to the affine function of the base function of the Bregman divergence, the structure of the Bregman-beta divergence $D_{\Phi}$ does not change, irrespective of the choice of the affine function of the base function. Hence, for simplicity, we add $x$ to $\Phi$~\eqref{basefn}, when $\alpha=1$. 

\begin{corollary}\label{legendre_betax}
Let us consider $\Phi$ with $\dom\Phi$ in \eqref{conditionLegendre} and $\Psi$ with $\dom\Psi$ in \eqref{conditionLegendreD}. Then, we have 
\begin{equation}\label{conjugateEq}
\Psi^* = \Phi
\end{equation}
where $\Psi^*(x) = \sup_{\xi} \inprod{x}{\xi} - \Psi(\xi)$ and constants are dropped. 
\end{corollary}
\begin{proof}
It is easy to check $\Phi^* = \Psi$ when $\alpha \in \{ 1, 2 \}$. Now, let us assume that $\alpha \not\in \{1,2\}$.
From the detail computations of the conjugate function $\Phi^*$ done in \citep{woo17}[Theorem 7], we have
$
\Phi^* = \Psi.
$
Since $\Phi$ is a convex function of Legendre type, we have $\Phi = \Phi^{**} = \Psi^*$. Additionally, we have that $(\Psi')^{-1}(x) = (\Psi^*)'(x) = \ln_{\alpha}(x)$ where $\Psi'(\xi) = \exp_{\alpha}(\xi).$ 
\end{proof}
In Figure \ref{fig:img2}, we plot the Bregman-Tweedie divergence $D_{\Psi}$ and the base function $\Psi$, and the corresponding Bregman-beta divergence $D_{\Phi}$ and the base function $\Phi$ as well. Even though $\Phi$ and $\Psi$ are strict convex functions on their interior of the domains, the Bregman divergences are not always convex in terms of the second variable.  For instance, Figure \ref{fig:img2} (c) shows that $D_{\Phi}(2|\mu)$ with $\alpha=2/3$ is a non-convex function in terms of $\mu$. 

\section{Bregman-Tweedie classification model\label{sec3}}
This Section introduces the Bregman-Tweedie loss function which is the regular Legendre transformation of the Bregman-Tweedie divergence. Instead of the equivalence class in \eqref{exexp}, we use $\exp_{\alpha,c}(x)$, the extended exponential function with a  scaling parameter $c$. For simplicity, we only consider $\alpha \in \{0, 1\} \cup \{(0,1) \cap \R_e\}$. Then, irrespective of the choice of $\alpha$ and $c$, we have $\dom(\exp_{\alpha,c}) = \R$. 

Let us start with the definition of the regular Legendre transformation of the Bregman-divergence associated with the convex function of Legendre type. 
\begin{definition}\label{def:legendreB}
Let $f : \dom f \rightarrow \R$ be a convex function of Legendre type and $\dom f =  int\dom f$. Additionally, let $D_f(z|x) : \dom f \times \dom f \rightarrow \R$ be the Bregman divergence associated with the convex function $f$ of Legendre type and $\eta + \dom f^* \subseteq \dom f^*$. Then, for all $x \in \dom f$, we have the regular Legendre transformation (of the Bregman-divergence associated with the convex function of Legendre type)
\begin{equation}\label{legendreB}
{\cal L}_f(\eta,x) = \underset{z \in dom f}{\arg\sup}\; \inprod{\eta}{z} - D_f(z | x).
\end{equation}
\end{definition}
In general, the convex function $f$ of Legendre type satisfies the following isomorphism~\citep{bauschke97}: 
$
f' : int\dom f \rightarrow int\dom f^*
$
where $(f')^{-1} = (f^*)'$ and $f^*(x) = \sup_{y} \inprod{x}{y} - f(y)$. With this isomorphism, we can show that the regular Legendre transformation of the Bregman divergence~\eqref{legendreB} has an additive structure. This is useful in analyzing the structure of the extended adaboosting~\citep{lafferty99}. \begin{theorem}
Let $\dom f = int\dom f$ and ${\cal L}_f(\eta,x)$ be the regular Legendre transformation of the Bregman divergence~\eqref{legendreB}. Then, for all $x \in \dom f$, we have 
\begin{equation}\label{additiveB}
{\cal L}_{f}(\eta_1, {\cal L}_{f}(\eta_2,x) ) = {\cal L}_f(\eta_1 + \eta_2,x)
\end{equation}
\end{theorem} 
\begin{proof}
From \eqref{legendreB}, for all $x \in \dom f$ and $\eta_1 + \dom f^* \subseteq \dom f^*$,  we have 
${\cal L}_f(\eta_1,x) = (f^*)'(\eta_1 + f'(x)) \in \dom f.$
Therefore, we have the following:
$${\cal L}_f(\eta_2 , {\cal L}_f(\eta_1,x)) = (f^*)'(\eta_2 + f'({\cal L}_f(\eta_1,x)))
= (f^*)'(\eta_1 + \eta_2 + f'(x))
= {\cal L}_f( \eta_1 + \eta_2 , x)$$
where $\eta_2 + \dom f^* \subseteq \dom f^*$.
\end{proof}
Now, let us consider the regular Legendre transformation of the Bregman-beta divergence, that is, the extended exponential loss function.
\begin{example}
Let us assume that $\alpha \in \{(0,1) \cap \R_e\} \cup \{0, 1 \}$ and consider 
\begin{equation}\label{PhiB}
{\cal L}_{\Phi}(\eta,x) = \underset{z \in \R}{\arg\sup}\; \inprod{\eta}{z} - D_{\Phi}(z|x)
\end{equation}
Then, we get
${\cal L}_{\Phi}(\eta,x) = \exp_{\alpha,c}(\eta + \ln_{\alpha,c}(x))$.
When $\alpha=1$, \eqref{PhiB} becomes the classic exponential loss commonly used in adaboosting. Let the observed data be $(x_i,y_i) \in \R^n \times \{-1,+1\}$. Then, for binary classification, we have $\eta = y_i h(x_i)$, where $h$ is a classifier. That is, $h(x) = \inprod{x}{w} + b$ for  linear classifier and $h(x) = \inprod{g(x)}{w}$ with $w \in \R^N$ and $g = (g_1,...,g_N)$ for boosting. Here, $g_i$ is the so-called weak classifiers. Hence, $\eta \in \R$ is recommended for the general classification model. This condition is satisfied when $\alpha \in \{(0,1) \cap \R_e\} \cup \{0, 1\}$. However, when $\alpha \in \{(0,1) \cap \R_e\} \cup \{ 0 \}$, $\exp_{\alpha}(y)$ is unbounded below.  As described in \citep{woo19}, by augmenting the Perceptron loss function, we have a connection to the higher-order hinge loss:
\begin{equation}
\max(0, {\cal L}_{\Phi}(\eta,x)) = \max(0,\exp_{\alpha,c}(\eta + \ln_{\alpha,c}(x))) 
\end{equation}
In fact, if $\alpha = \frac{2k}{2k+1}$ and $c_{\alpha} = -1$ then we obtain
$
\max(0,\exp_{\alpha,c}(-y)) = c\max\left(0, 1 - y\right)^{2k+1}. 
$
This becomes the well-known hinge loss when $k=0$ (i.e., $\alpha=0$).
\end{example}
In the following, instead of the Bregman-beta divergence, we make use of the Bregman-Tweedie divergence for the classification loss function. 
\begin{theorem}
Let $\alpha  \in \{(-\infty,1) \cap \R_e\} \cup [1,2)$ and  $c \in \R_{++}$. Then, we have
\begin{equation}\label{conjL}
{\cal L}_{\Psi}(c,x) = \underset{z \in dom\Psi}{\arg\sup}\; \inprod{c}{z}  - D_{\Psi}(z|x)
\end{equation}
where $\Psi$~\eqref{basefnD} is the convex function of Legendre type. Then we have 
\begin{equation}\label{xx}
{\cal L}_{\Psi}(c,x) =
\left\{ 
\begin{array}{l}
 \ln(1 + \exp(x)) \qquad\quad\;\; \hbox{ if } \alpha = 1\\
\ln_{\alpha,c}(c + \exp_{\alpha,c}(x)) 
\quad\;\;\hbox{otherwise }\\
\end{array}
\right.
\end{equation}
where $x \in \dom \Psi$ and $c  + \dom\Psi^* \subseteq \dom \Psi^*$. 
Note that $\dom \Psi$ is defined as
\begin{equation}
\left\{
\begin{array}{l}
\alpha < 1, \alpha \in \R_e \hskip 1cm \hbox{ and } \hskip 0.5cm \dom\Psi = \R,\\
\alpha = 1, \hskip 2.2cm \hbox{ and } \hskip 0.5cm \dom\Psi = \R,\\
 1 <  \alpha < 2 \hskip 1.6cm \hbox{ and }  \hskip 0.5cm \dom\Psi  = \R_{<c_{\alpha}},\\
\end{array}
\right.
\end{equation}
When $\alpha \in \{ (0,1) \in \R_e \} \cup \{0,1\}$, we call \eqref{xx} as the Bregman-Tweedie loss function.
\end{theorem}
\begin{proof}
From 
$D_{\Psi}(z|x) = \Psi(z) - \Psi(x) - \inprod{\Psi'(x)}{z-x}$, we have 
\begin{equation*}
{\cal L}_{\Psi}(c,x) = \Phi' (c + \Psi'(x)) = \ln_{\alpha,c}(c + \exp_{\alpha,c}(x))
\end{equation*}
where $\exp_{\alpha,c}(x) = ((1-\alpha)[x])^{1/(1-\alpha)}$ and $[x] = x- c_{\alpha}$.
When $\alpha < 1$ and $\alpha \in \R_e$,  $\dom(\exp_{\alpha,c}) = \dom(\ln_{\alpha,c}) = \R$ and thus ${\cal L}_{\Psi}(c,x)$ is well defined for all $c \in \R_{++}$. When $\alpha=1$, we get the classic logistic loss function ${\cal L}(c,x) = \ln(1 + \exp(x))$. Now, let us consider $1<\alpha<2$. From \eqref{conditionLegendreD} and the equivalence class $[x] = x -c_{\alpha}$, we have $[x] = x - c_{\alpha}<0$ and thus
$\dom\Psi =  \R_{< c_{\alpha}}$. In case of $1<\alpha<2$ and $\alpha \in \R_e$, we can select $[x] = x - c_{\alpha} > 0$. However, since $c_{\alpha}>0$, $0 \not\in \dom\Psi = \R_{>c_{\alpha}}$. It does not  satisfy classification-calibration~\citep{bartlett06} condition at all and thus this region is not useful for classification.
\end{proof}
When $1<\alpha < 2$, \eqref{xx} is  the extended logistic loss function in~\citep{woo19}. In the following, we assume that \eqref{xx} with $\alpha \in \{(0,1) \cap \R_e\} \cup \{0,1\}$ as the Bregman-Tweedie loss function. Note that \eqref{xx} with $\alpha=0$ becomes the unhinge loss function~\citep{rooyen15} and \eqref{xx} with $\alpha=1$ becomes the famous logistic loss function. As observed in Figure \ref{fig:img3} (a), when $c_{\alpha}=-1$, \eqref{xx} with $\alpha \in (0,1) \cap \R_e$ behaves like the hinge loss function (H-Bregman). On the other hand, if we set $c=1$ then we get logistic loss like function (L-Bregman). See Figure \ref{fig:img3} (c) for more details. Regarding gradients of the Bregman-Tweedie loss functions, in the following corollary, we introduce the gradient of the Bregman-Tweedie loss function. 
\begin{corollary}
Let $c \in \R_{++}$. Then
the gradient of the Bregman-Tweedie loss function ${\cal L}_{\Psi}(c,x)$ with $\alpha \in (0,1) \cap \R_e$ becomes
$
\frac{\partial}{\partial x}{\cal L}_{\Psi}(c,x) = {\cal L}_{\Psi}'(c,x) = -\left(\frac{\exp_{\alpha,c}(x)}{c + \exp_{\alpha,c}(x)}\right)^{\alpha} 
$
and the domain of it is 
\begin{equation}\label{conditionLegendreDx}
\dom {\cal L}_{\Psi}' = \R_{> 2c_{\alpha}}.
\end{equation}
\end{corollary}
\begin{proof}
Let $0< \alpha < 1$ and $\alpha \in \R_e$, then we have rather complicated domain restriction. In fact, since $\Psi$ is the convex function of Legendre, $\Psi'(x) = \exp_{\alpha,c}(x)$ is monotone increasing on its domain $\R$. Thus, there is $a \in \R$ satisfying $c + \exp_{\alpha,c}(a) = 0.$ By simple calculation, we have $a = 2c_{\alpha}$. Thus, ${\cal L}_{\Psi}'(c,2c_{\alpha})$ is undefined and $\dom{\cal L}_{\Psi}'$ is restricted to $\R_{>2c_{\alpha}}$. Notice that, we can select $\R_{<2c_{\alpha}}$ as $\dom{\cal L}'_{\Psi}$. In this case, since $c_{\alpha}<0$, $0 \not\in \R_{<2c_{\alpha}}$ and thus this region is not useful for classification~\citep{bartlett06}.
\end{proof}
Figure \ref{fig:img3} (b) and (d) demonstrate the gradient of the Bregman-Tweedie loss function ${\cal L}_{\Psi}(c,-x) = \ln_{\alpha,c}(c + \exp_{\alpha,c}(-x))$. As observed in Figure \ref{fig:img3} (b), when $c_{\alpha}=-1$, we have $\dom {\cal L}_{\Psi}'(c,-x) = \R_{<2}$ and does not depends on $\alpha$. However, when $\alpha=20/101$ and $c=1$ (Figure \ref{fig:img3} (d)), we have $2c_{\alpha} = -101/41 = -2.46$ and thus $\dom {\cal L}_{\Psi}'(1,-x) = \R_{<2.46}$. That is, $\dom{\cal L}_{\Psi}'$ depends on the parameter $\alpha$. Due to the strict restriction of the domain of the gradient of the Bregman-Tweedie loss function, the role of this function for classification is rather limited. However, by simply restricting the domain of the data set, we can overcome this drawback. Let $(x_i,y_i) \in {\cal X} \times \{-1,+1\}$ be training data of the binary classification problem. Note that ${\cal X} = \{ x \in \R \;|\; \norm{x}_{1} <  B_X \}$ where $B_X \in \R_{++}$ is a constant. The corresponding linear decision boundary is given as ${\cal H} = \{ h(x) = \inprod{w}{x} + b \;|\; (w,b) \in {\cal W} \}$ with ${\cal W} = \{(w,b) \in \R^n \times \R  \;|\; \norm{(w,b)}_{\infty} < B_W \}$ and $B_W \in \R_{++}$ is a constant. Hence, we have $\abs{h(x)} \le B$ for some appropriate constant $B\; (\ge (B_X + 1)B_W)$. When $B < -2c_{\alpha}$ for $\alpha \in (0,1) \cap \R_e$, we can use the Bregman-Tweedie loss function ${\cal L}_{\Psi}(c,-x)$ for the classification problem. Let us consider 
\begin{equation}\label{datanorm}
\abs{y_i(\inprod{w}{x_i}+b)} \le (B_{X} + 1)B_W \le \rho\abs{c_{\alpha}}
\end{equation}
where $\rho < 2$. Since we minimize with respect to $(w,b)$, it is not easy to use the bound in \eqref{datanorm}. Hence, we rescale the given data by $x_i \define x_i/(B_X+1).$ Then, we can set $B_W= \rho\abs{c_{\alpha}}$ with $\rho \in (1,2)$. 

Now, let us introduce the {\it Bregman-Tweedie classification model}:
\begin{equation}\label{binClass}
\min_{ (w,b) \in {\cal W}}\; {\cal H}(w,b) + \lambda \norm{w}_2^2
\end{equation}
where $
{\cal W} = \{ (w,b) \in \R^n \times \R \;|\; \norm{(w,b)}_{\infty} < \rho \abs{c_{\alpha}} \}
$
and 
the Bregman-Tweedie loss function is defined as
\begin{equation}\label{binClassloss}
{\cal H}(w,b) = \sum_{i=1}^n\;  \ln_{\alpha,c}(c + \exp_{\alpha,c}(-y_i(\inprod{w}{x_i}+b))).
\end{equation}
The numerical experiments with the proposed Bregman-Tweedie classification model~\eqref{binClass} are given in the following Section \ref{sec4}.

\begin{figure*}[t]
\centering
\includegraphics[width=5.5in]{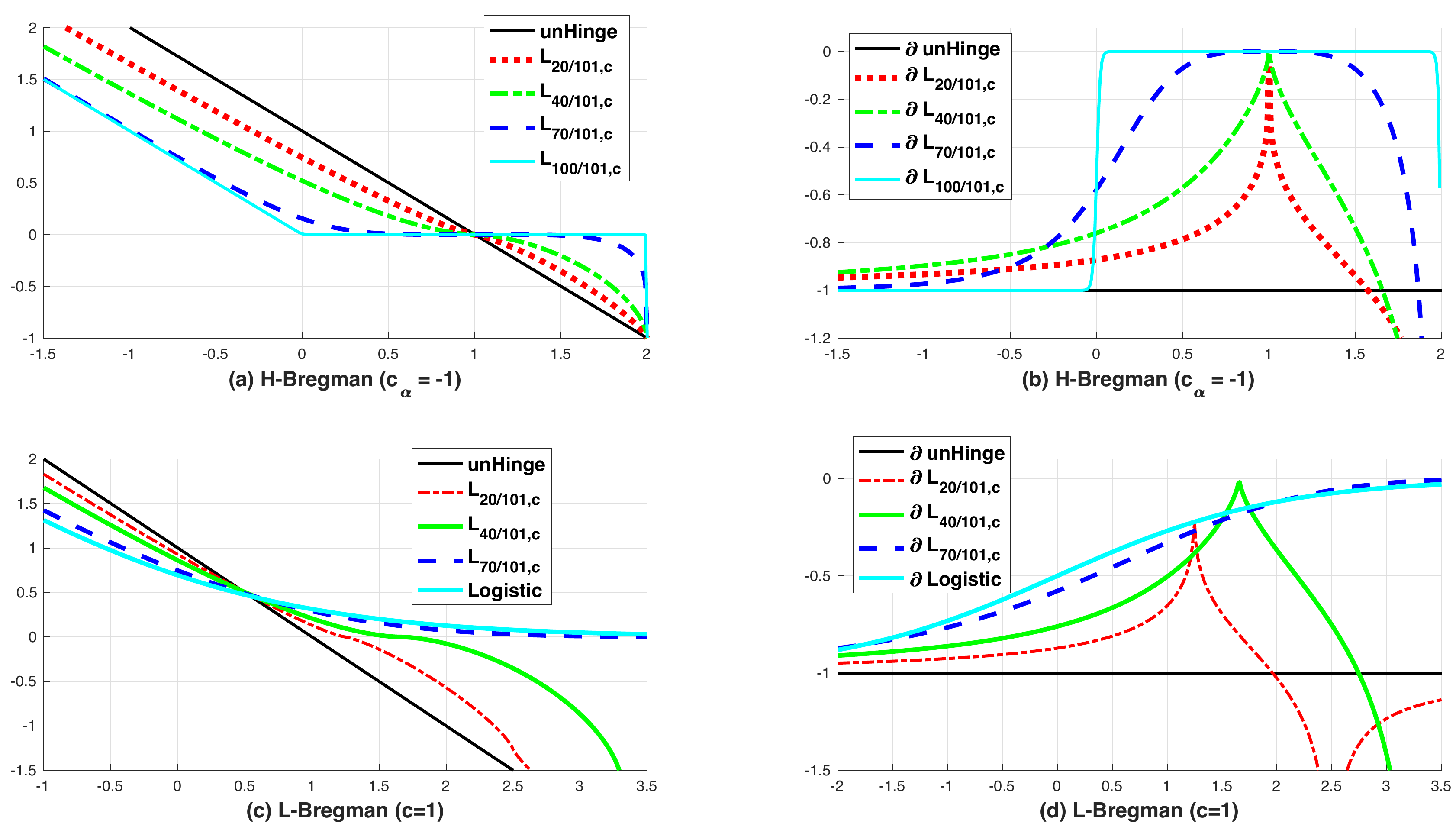} 
\caption{Graphs of the Bregman-Tweedie loss function ${\cal L}_{\Psi}(c,-x)$; (a) The Bregman-Tweedie loss with $c_{\alpha}=-1$ ($\alpha = 20/101,40/101,70/101,100/101$). (b) The gradient of the Bregman-Tweedie loss in (a). We have $\dom {\cal L}'_{\Psi}(c,-x) = \R_{<2}$. (c) Bregman-Tweedie loss with with $c=1$ ($\alpha = 20/101,40/101,70/101$). (d) The gradient of the Bregman-Tweedie loss in (c). When $\alpha=20/101$, we have $2c_{\alpha} = -2.46$ and thus $\dom {\cal L}'_{\Psi}(c,-x) = \R_{<2.46}$. 
}
\label{fig:img3}
\end{figure*}

 \section{Numerical experiments for Bregman-Tweedie logistic regression model\label{sec4}} 
This Section compares the performance of the proposed Bregman-Tweedie classification model~\eqref{binClass} with the logistic regression and SVM for the problem of learning linear decision boundary. 

\begin{figure*}[t]
\centering
\includegraphics[width=5.5in]
{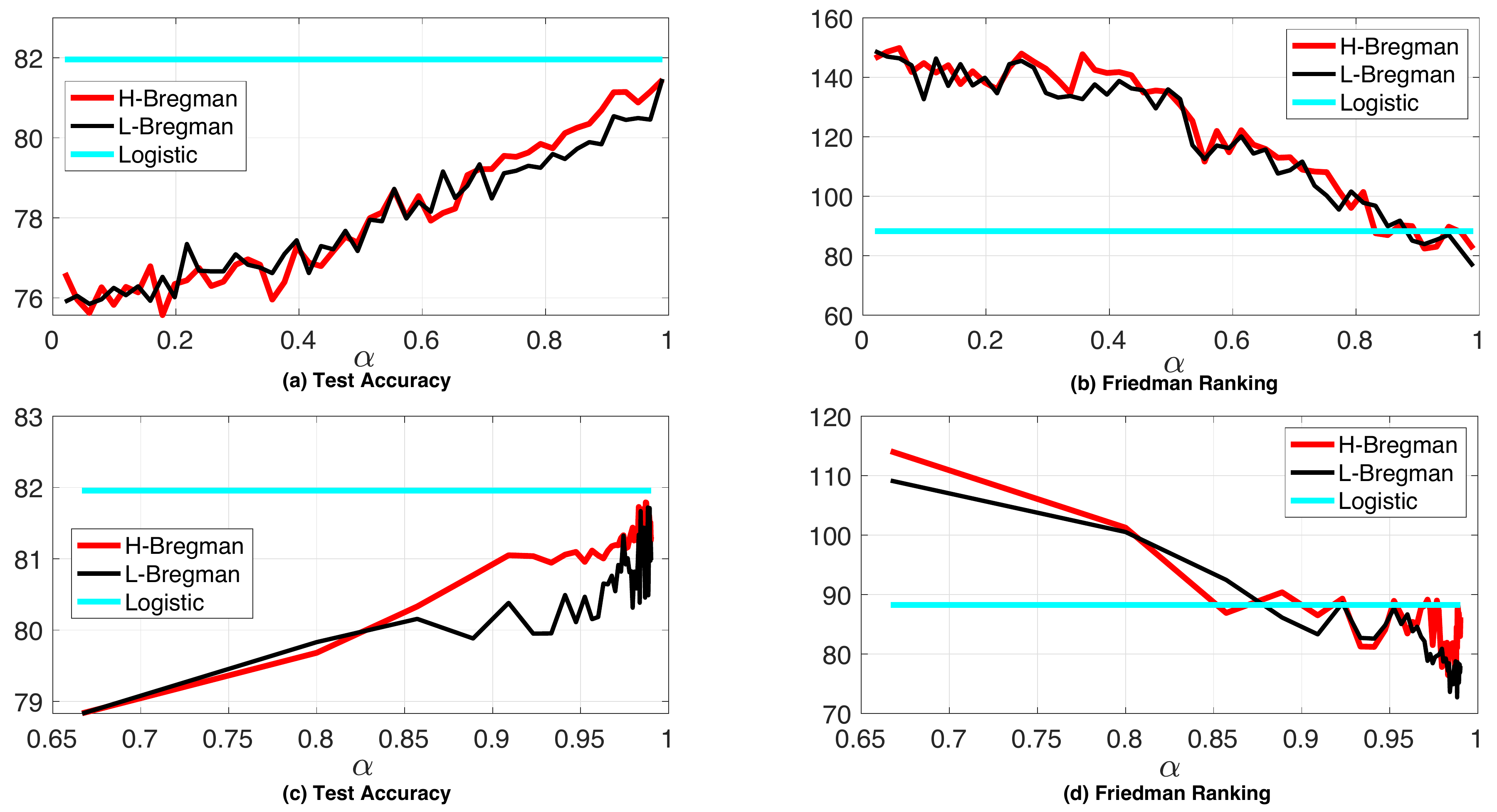}
\caption{A comparison of the Bregman-Tweedie classification model (H-Bregman ($c_{\alpha}=1$) and L-Bregman ($c=1$)) with logistic regression~\citep{fan08}. In (a) and (b), we set $\alpha  = 2k/101$ with $k=1,...,50$. In (c) and (d), we set $\alpha = 2k/(2k+1)$ with $k=1,...,50$. Note that (a) and (c) report the test classification accuracy. (b) and (d) report the Friedman  ranking. In terms of Friedman ranking, the Bregman-Tweedie classification model shows better performance than logistic regression when $\alpha \approx 1$. Note that H-Bregman is better than L-Bregman in terms of test classification accuracy when $\alpha>0.85$. L-Bregman is better than H-Bregman in terms of Friedman ranking when $\alpha>0.95$.}
\label{fig:imgx1}
\end{figure*}

For the minimization of the Bregman-Tweedie classification model~\eqref{binClass}, we use a limited-memory projected quasi-Newton ({\it minConf\_PQN} in \citep{mark19}). This algorithm is a typical constraint optimization algorithm implemented with the MATLAB. We use the famous LIBLINEAR package~\citep{fan08} for the benchmark of the proposed classification model~\eqref{binClass}. Among various linear classification models in  LIBLINEAR, we select typical models; logistic regression and higher-order SVM~
(the first-order SVM and the second-order SVM (i.e., L2SVM)). For logistic regression, we use the primal formulation ($s=0$). For SVM, we use the dual formulation ($s=3$). For L2SVM, we use the primal formulation ($s=2$). We also use the bias term in LIBLINEAR ($B=1$). All models have $\ell_2$-regularization term. 
As regards the regularization parameter $\lambda$, we simply use the following parameter space of $\lambda$ as recommended in the LIBSVM~\citep{chang11}.
\begin{equation}\label{lambdaD}
\lambda = 2^{b}, \; b = -14, -13, -12, ... , 4, 5
\end{equation} 
In the models of LIBLINEAR, the regularization parameter is located on the loss function and thus we use $\lambda^{-1}$ of \eqref{lambdaD} for the regularization parameter of them. For the best regularization parameter $\lambda$, we use four-fold cross validation~\citep{delgado14}. 

In terms of parameter space of the Bregman-Tweedie  loss function, we need to select not only the regularization parameter $\lambda$ but also the model parameter $\alpha$ and $c$. We categorize the Bregman-Tweedie classification model (for simplicity, we only consider $\alpha \in (0,1) \cap \R_e$) into two different sub-models (H-Bregman and L-Bregman). The {H-Bregman} is the hinge-like Bregman-Tweedie classification model ($c_{\alpha}=1$) and the  {L-Bregman} is the logistic-like Bregman-Tweedie classification model ($c=1$).  

For the benchmark dataset, we use the well-organized datasets in~\citep{delgado14}, while reporting the performance of the Bregman-Tweedie classification  models. They are pre-processed and normalized in each feature dimension with mean zero and variance one. Additionally, each data $x_i$ is normalized by $x_i/(B_X+1)$ as mentioned in \eqref{datanorm}. Also, we set $\rho=1.5$ for ${\cal W} = \{ (w,b) \in \R^n \times \R \;|\; \norm{(w,b)}_{\infty} < \rho \abs{c_{\alpha}}\}$. The raw format of each data is available in UCI machine learning repository. Note that, as commented in~\citep{wainberg16}, we reorganize the dataset in \citep{delgado14}. First, each dataset is separated into the training and testing data set which are not overlapped. Each training data set is randomly shuffled for four-fold cross validation. Among the dataset in~\citep{delgado14}, we use fifty-one two-class classification datasets after removing ambiguous dataset in terms of data splitting strategy. In Table~\ref{tableFullData}, we list up all information of datasets such as number of instances, number of train data, number of test data, feature dimension, and number of classes.  

The whole experiments are run five times and the averaged test score of each dataset is reported in Table \ref{tableFullDatax}. In each experiment, the best regularization parameters are chosen through the four-fold cross-validation. With the chosen best parameter, we minimize the proposed Bregman-Tweedie classification model~\eqref{binClass} with the whole training data in Table \ref{tableFullData} to find the hyperplane $(w,b) \in \R^n \times \R$. Then we evaluate the performance of each classification model with test data in Table \ref{tableFullData}. For more details on cross-validation-based approach, see \citep{chang11}. 

In Figure \ref{fig:imgx1}, we plot the test classification accuracy ((a) and (c)) and Friedman ranking ((b) and (d)) of the Bregman-Tweedie classification model (H-Bregman and L-Bregman).  We set $\alpha=\frac{2k}{101}$ with $k=1,...,50$ for Figure \ref{fig:imgx1} (a) and (b) and $\alpha=\frac{2k}{2k+1}$ with $k=1,...,50$ for Figure \ref{fig:imgx1} (c) and (d). Note that the performance evaluation at each $\alpha$ is the average score of the five times repeated test accuracy of all dataset in Table \ref{tableFullData}. As $\alpha \rightarrow 1$, the proposed Bregman-Tweedie classification model (H-Bregman and L-Bregman) shows better performance. Especially, in terms of Friedman ranking, the proposed model obtains better performance than the classic logistic regression when $\alpha \approx 1$. Interestingly, H-Bregman is better than L-Bregman with respect to the test classification accuracy and L-Bregman is better than H-Bregman with respect to the Friedman ranking. Among various $\alpha$ in Figure \ref{fig:imgx1}, we select five $\alpha$ having best classification accuracy. That is, $\alpha = 58/59$(HB1), $68/69$(HB2), $76/77$(HB3), $78/79$(HB4), $90/91$(HB5) for H-Bregman and $\alpha = 62/63$(LB1), $70/71$(LB2), $80/81$(LB3), $84/85$(LB4), $92/93$(LB5) for L-Bregman. All numerical results for each dataset with HB1-HB5 and LB1-LB5 are summarized in Table \ref{tableFullDatax}. In terms of Friedman ranking, LB4 ($\alpha=84/85$) shows the best performance. However, the logistic regression in LIBLINEAR~\citep{fan08} obtains the best performance in terms of test classification accuracy.

\begin{table}
\tiny
\centerline{
\begin{tabular}{|l||c|c|c|c|c|}
\hline
&\textbf{Instance}&\textbf{Train}&\textbf{Test}&\textbf{Feature dim}&\textbf{Class}\\\hline\hline
\textbf{acute-inflammation  }&120&60&60&6&2\\\hline
\textbf{acute-nephritis  }&120&60&60&6&2\\\hline
\textbf{adult  }&48842&32561&16281&14&2\\\hline
\textbf{balloons  }&16&8&8&4&2\\\hline
\textbf{bank  }&4521&2261&2260&16&2\\\hline
\textbf{blood  }&748&374&374&4&2\\\hline
\textbf{breast-cancer  }&286&143&143&9&2\\\hline
\textbf{breast-cancer-wisc  }&699&350&349&9&2\\\hline
\textbf{breast-cancer-wisc-diag  }&569&285&284&30&2\\\hline
\textbf{breast-cancer-wisc-prog  }&198&99&99&33&2\\\hline
\textbf{chess-krvkp  }&3196&1598&1598&36&2\\\hline
\textbf{congressional-voting  }&435&218&217&16&2\\\hline
\textbf{conn-bench-sonar-mines-rocks  }&208&104&104&60&2\\\hline
\textbf{connect-4  }&67557&33779&33778&42&2\\\hline
\textbf{credit-approval  }&690&345&345&15&2\\\hline
\textbf{cylinder-bands  }&512&256&256&35&2\\\hline
\textbf{echocardiogram  }&131&66&65&10&2\\\hline
\textbf{fertility  }&100&50&50&9&2\\\hline
\textbf{haberman-survival  }&306&153&153&3&2\\\hline
\textbf{heart-hungarian  }&294&147&147&12&2\\\hline
\textbf{hepatitis  }&155&78&77&19&2\\\hline
\textbf{hill-valley  }&606&303&303&100&2\\\hline
\textbf{horse-colic  }&368&300&68&25&2\\\hline
\textbf{ilpd-indian-liver  }&583&292&291&9&2\\\hline
\textbf{ionosphere  }&351&176&175&33&2\\\hline
\textbf{magic  }&19020&9510&9510&10&2\\\hline
\textbf{miniboone  }&130064&65032&65032&50&2\\\hline
\textbf{molec-biol-promoter  }&106&53&53&57&2\\\hline
\textbf{mammographic  }&961&481&480&5&2\\\hline
\textbf{mushroom  }&8124&4062&4062&21&2\\\hline
\textbf{musk-1  }&476&238&238&166&2\\\hline
\textbf{musk-2  }&6598&3299&3299&166&2\\\hline
\textbf{oocytes-merluccius-nucleus-4d  }&1022&511&511&41&2\\\hline
\textbf{oocytes-trisopterus-nucleus-2f  }&912&456&456&25&2\\\hline
\textbf{ozone  }&2536&1268&1268&72&2\\\hline
\textbf{parkinsons  }&195&98&97&22&2\\\hline
\textbf{pima  }&768&384&384&8&2\\\hline
\textbf{pittsburg-bridges-T-OR-D  }&102&51&51&7&2\\\hline
\textbf{planning  }&182&91&91&12&2\\\hline
\textbf{ringnorm  }&7400&3700&3700&20&2\\\hline
\textbf{spambase  }&4601&2301&2300&57&2\\\hline
\textbf{spect  }&265&79&186&22&2\\\hline
\textbf{spectf  }&267&80&187&44&2\\\hline
\textbf{statlog-australian-credit  }&690&345&345&14&2\\\hline
\textbf{statlog-german-credit  }&1000&500&500&24&2\\\hline
\textbf{statlog-heart  }&270&135&135&13&2\\\hline
\textbf{tic-tac-toe  }&958&479&479&9&2\\\hline
\textbf{titanic  }&2201&1101&1100&3&2\\\hline
\textbf{trains  }&10&5&5&29&2\\\hline
\textbf{twonorm  }&7400&3700&3700&20&2\\\hline
\textbf{vertebral-column-2clases  }&310&155&155&6&2\\\hline
\end{tabular}
}
\caption{The list of all two-class datasets used in this article. This is a corrected version of dataset available in \citep{delgado14} based on~\citep{wainberg16}. The most dataset in this Table is available in UCI repository as raw formats.}\label{tableFullData}
\end{table}

\begin{table*}
\centerline{
\begin{tiny}\begin{tabular}{|l||c|c|c|c|c||c|c|c|c|c||c|c|c|}
\hline
&\textbf{HB1}&\textbf{HB2}&\textbf{HB3}&\textbf{HB4}&\textbf{HB5}&\textbf{LB1}&\textbf{LB2}&\textbf{LB3}&\textbf{LB4}&\textbf{LB5}&\textbf{Logistic}&\textbf{SVM}&\textbf{L2SVM}\\\hline\hline
\textbf{acute-inflammation  }&{\bf 100.00}&{\bf 100.00}&{\bf 100.00}&{\bf 100.00}&{\bf 100.00}&{\bf 100.00}&{\bf 100.00}&{\bf 100.00}&{\bf 100.00}&{\bf 100.00}&{\bf 100.00}&{\bf 100.00}&{\bf 100.00}\\\hline
\textbf{acute-nephritis  }&{\bf 100.00}&{\bf 100.00}&{\bf 100.00}&{\bf 100.00}&{\bf 100.00}&{\bf 100.00}&{\bf 100.00}&{\bf 100.00}&{\bf 100.00}&{\bf 100.00}&{\bf 100.00}&{\bf 100.00}&{\bf 100.00}\\\hline
\textbf{adult  }&84.15&84.15&84.16&84.16&84.15&84.15&84.16&84.17&84.16&84.15&84.29&{\bf 84.32}&84.09\\\hline
\textbf{balloons  }&87.50&87.50&87.50&87.50&87.50&{\bf 100.00}&87.50&87.50&{\bf 100.00}&{\bf 100.00}&87.50&87.50&87.50\\\hline
\textbf{bank  }&89.03&89.03&{\bf 89.07}&{\bf 89.07}&{\bf 89.07}&89.03&89.03&89.03&89.03&89.03&88.81&88.50&88.81\\\hline
\textbf{blood  }&{\bf 76.20}&{\bf 76.20}&{\bf 76.20}&{\bf 76.20}&{\bf 76.20}&{\bf 76.20}&{\bf 76.20}&{\bf 76.20}&{\bf 76.20}&{\bf 76.20}&75.94&{\bf 76.20}&75.67\\\hline
\textbf{breast-cancer  }&72.03&72.03&{\bf 72.73}&{\bf 72.73}&70.63&72.03&72.03&72.03&72.03&72.03&71.33&69.23&71.33\\\hline
\textbf{breast-cancer-wisc  }&96.56&96.56&{\bf 96.85}&{\bf 96.85}&96.56&96.56&96.56&96.56&96.56&96.56&96.56&96.50&96.56\\\hline
\textbf{breast-cancer-wisc-d}&{\bf 98.94}&98.59&96.48&98.59&98.59&98.59&98.59&98.59&98.59&98.59&97.89&97.54&97.89\\\hline
\textbf{breast-cancer-wisc-p}&79.80&78.79&{\bf 81.82}&81.21&{\bf 81.82}&79.80&79.80&79.80&79.80&79.80&70.71&75.76&76.77\\\hline
\textbf{chess-krvkp  }&95.68&95.74&95.99&96.06&96.25&95.74&95.81&96.18&96.25&96.31&{\bf 96.62}&96.47&96.56\\\hline
\textbf{congressional-voting}&58.53&60.83&60.83&60.37&60.37&61.29&61.29&61.29&61.29&61.29&55.76&{\bf 61.75}&57.14\\\hline
\textbf{conn-bench-sonar-}&{\bf 77.88}&76.15&{\bf 77.88}&75.96&75.96&75.00&75.00&75.00&75.00&75.00&74.04&{\bf 77.88}&74.04\\\hline
\textbf{connect-4  }&{\bf 75.48}&75.46&75.47&{\bf 75.48}&75.47&75.47&75.47&75.47&{\bf 75.48}&75.47&75.47&75.38&75.41\\\hline
\textbf{credit-approval  }&{\bf 89.28}&{\bf 89.28}&88.70&88.70&88.70&{\bf 89.28}&{\bf 89.28}&{\bf 89.28}&{\bf 89.28}&{\bf 89.28}&88.12&87.54&87.83\\\hline
\textbf{cylinder-bands  }&65.23&65.23&65.23&65.23&70.70&65.23&65.23&65.23&65.23&65.23&74.61&{\bf 76.17}&74.61\\\hline
\textbf{echocardiogram  }&80.00&80.00&78.46&78.46&80.00&80.00&80.00&80.00&80.00&80.00&83.08&{\bf 87.69}&84.62\\\hline
\textbf{fertility  }&86.00&86.00&{\bf 88.00}&{\bf 88.00}&{\bf 88.00}&86.00&86.00&86.00&86.00&86.00&{\bf 88.00}&{\bf 88.00}&86.00\\\hline
\textbf{haberman-survival  }&{\bf 74.51}&73.20&73.20&73.20&73.20&73.86&73.20&73.20&73.86&73.20&73.86&73.73&73.86\\\hline
\textbf{heart-hungarian  }&{\bf 88.44}&87.76&87.76&87.76&87.07&{\bf 88.44}&{\bf 88.44}&{\bf 88.44}&{\bf 88.44}&{\bf 88.44}&87.76&83.67&87.07\\\hline
\textbf{hepatitis  }&{\bf 77.92}&{\bf 77.92}&{\bf 77.92}&{\bf 77.92}&{\bf 77.92}&{\bf 77.92}&{\bf 77.92}&{\bf 77.92}&{\bf 77.92}&{\bf 77.92}&{\bf 77.92}&72.99&76.62\\\hline
\textbf{hill-valley  }&57.76&57.23&57.16&57.23&56.77&57.10&57.10&57.10&57.29&57.29&{\bf 80.20}&65.54&66.34\\\hline
\textbf{horse-colic  }&{\bf 88.24}&{\bf 88.24}&{\bf 88.24}&{\bf 88.24}&{\bf 88.24}&{\bf 88.24}&{\bf 88.24}&{\bf 88.24}&{\bf 88.24}&{\bf 88.24}&{\bf 88.24}&86.76&{\bf 88.24}\\\hline
\textbf{ilpd-indian-liver  }&71.48&72.16&71.48&71.82&72.16&72.16&72.16&72.16&72.16&72.16&70.45&71.48&{\bf 72.51}\\\hline
\textbf{ionosphere  }&87.43&86.86&85.14&84.00&85.71&86.86&86.86&86.86&86.86&86.86&{\bf 88.57}&{\bf 88.57}&86.86\\\hline
\textbf{magic  }&79.26&79.20&79.19&79.20&79.18&79.22&79.22&79.22&79.22&79.23&79.10&{\bf 79.64}&78.99\\\hline
\textbf{miniboone  }&87.06&86.59&87.15&87.26&87.05&87.23&87.32&87.36&87.41&87.32&90.36&{\bf 90.44}&87.86\\\hline
\textbf{molec-biol-promoter  }&73.96&75.47&76.98&75.85&76.60&75.47&75.47&75.47&75.47&75.47&{\bf 78.49}&75.47&76.98\\\hline
\textbf{mammographic  }&{\bf 83.75}&{\bf 83.75}&83.67&83.67&83.67&{\bf 83.75}&{\bf 83.75}&{\bf 83.75}&{\bf 83.75}&{\bf 83.75}&83.71&83.33&82.92\\\hline
\textbf{mushroom  }&94.50&94.48&94.48&94.46&94.46&94.50&94.48&94.46&94.46&94.46&94.46&{\bf 97.69}&93.99\\\hline
\textbf{musk-1  }&83.45&82.69&81.68&82.27&81.51&82.02&81.93&81.68&81.68&81.93&82.94&{\bf 84.62}&83.53\\\hline
\textbf{musk-2  }&90.74&91.82&92.11&92.21&92.80&90.95&91.89&92.28&92.50&92.71&94.74&{\bf 95.02}&94.85\\\hline
\textbf{oocytes-merluccius-  }&78.71&79.22&78.79&79.22&78.90&79.30&79.45&79.14&79.69&79.30&82.74&80.98&{\bf 82.97}\\\hline
\textbf{oocytes-trisopterus-  }&79.39&80.00&79.82&79.43&80.44&79.87&80.04&80.13&80.04&80.26&78.73&{\bf 80.61}&79.17\\\hline
\textbf{ozone  }&{\bf 97.16}&{\bf 97.16}&{\bf 97.16}&{\bf 97.16}&97.13&{\bf 97.16}&{\bf 97.16}&97.13&97.13&97.13&97.15&97.10&{\bf 97.16}\\\hline
\textbf{parkinsons  }&81.44&81.44&83.92&83.92&82.06&82.27&82.27&82.27&82.27&82.27&82.47&84.12&{\bf 84.54}\\\hline
\textbf{pima  }&76.30&76.15&76.15&76.35&{\bf 76.46}&75.94&75.94&75.94&75.94&75.94&76.30&75.78&75.73\\\hline
\textbf{pittsburg-bridges-T-}&88.24&88.24&88.24&88.63&88.24&88.24&88.24&88.24&88.24&88.24&{\bf 90.20}&86.27&88.63\\\hline
\textbf{planning  }&{\bf 71.43}&71.21&71.21&70.99&70.99&{\bf 71.43}&{\bf 71.43}&{\bf 71.43}&{\bf 71.43}&{\bf 71.43}&65.05&{\bf 71.43}&65.27\\\hline
\textbf{ringnorm  }&77.73&77.77&77.76&77.77&77.76&{\bf 77.79}&{\bf 77.79}&{\bf 77.79}&{\bf 77.79}&{\bf 77.79}&76.87&77.48&77.09\\\hline
\textbf{spambase  }&93.04&{\bf 93.06}&93.03&92.99&92.98&{\bf 93.06}&{\bf 93.06}&{\bf 93.06}&{\bf 93.06}&{\bf 93.06}&92.22&92.79&92.20\\\hline
\textbf{spect  }&61.29&62.15&61.18&61.18&61.29&60.97&60.97&60.97&60.97&60.97&65.05&{\bf 66.13}&61.83\\\hline
\textbf{spectf  }&48.24&48.66&{\bf 49.09}&46.52&48.13&48.24&48.24&48.24&48.24&48.24&45.03&44.81&48.45\\\hline
\textbf{statlog-australian-}&67.59&{\bf 67.83}&{\bf 67.83}&{\bf 67.83}&{\bf 67.83}&{\bf 67.83}&{\bf 67.83}&{\bf 67.83}&{\bf 67.83}&{\bf 67.83}&66.96&{\bf 67.83}&66.96\\\hline
\textbf{statlog-german-}&{\bf 77.16}&76.24&76.20&76.04&76.00&77.00&77.00&77.00&77.00&77.00&{\bf 77.16}&75.52&76.92\\\hline
\textbf{statlog-heart  }&87.26&88.30&{\bf 88.74}&{\bf 88.74}&87.41&88.15&88.15&88.15&88.15&88.15&87.41&87.26&88.15\\\hline
\textbf{tic-tac-toe  }&{\bf 97.91}&{\bf 97.91}&{\bf 97.91}&{\bf 97.91}&{\bf 97.91}&{\bf 97.91}&{\bf 97.91}&{\bf 97.91}&{\bf 97.91}&{\bf 97.91}&{\bf 97.91}&{\bf 97.91}&{\bf 97.91}\\\hline
\textbf{titanic  }&{\bf 77.55}&{\bf 77.55}&{\bf 77.55}&{\bf 77.55}&{\bf 77.55}&{\bf 77.55}&{\bf 77.55}&{\bf 77.55}&{\bf 77.55}&{\bf 77.55}&{\bf 77.55}&{\bf 77.55}&{\bf 77.55}\\\hline
\textbf{trains  }&76.00&76.00&76.00&{\bf 80.00}&64.00&60.00&60.00&60.00&60.00&60.00&60.00&60.00&60.00\\\hline
\textbf{twonorm  }&97.76&97.74&97.71&97.71&97.71&{\bf 97.78}&{\bf 97.78}&{\bf 97.78}&{\bf 97.78}&{\bf 97.78}&97.68&97.58&97.51\\\hline
\textbf{vertebral-column-2c}&83.10&81.03&81.55&81.55&82.58&82.58&82.58&82.58&82.58&82.58&{\bf 83.74}&81.16&81.03\\\hline\hline
\textbf{Mean}&81.73&81.70&81.79&81.79&81.60&81.67&81.44&81.44&81.72&81.71&{\bf 81.96}&81.92&81.66\\\hline
\textbf{Friedman Ranking}& 6.73&7.21&7.08&6.82&7.73&6.63&6.68&6.81&{\bf 6.25}&6.46&7.19&7.40&8.03\\\hline
\end{tabular}
\end{tiny}}
\caption{A comparison of H-Bregman (HB1:$\alpha=58/59$, HB2:$\alpha=68/69$, HB3:$\alpha=76/77$, HB4:$\alpha=78/79$, HB5:$\alpha=90/91$), L-Bregman (LB1:$\alpha=62/63$, LB2:$\alpha=70/71$, LB3:$\alpha=80/81$, LB4:$\alpha=84/85$, LB5:$\alpha=92/93$), and LIBLINEAR (logistic regression, SVM, and L2SVM). In terms of Friedman ranking, LB4 shows the best performance. On the other hand, in terms of test classification accuracy, logistic regression in LIBLINEAR~\citep{fan08} shows the best performance.}\label{tableFullDatax}
\end{table*}

 \section{Conclusion\label{sec6}}
In this article, we have introduced the extended exponential function and the high-level structure based on this function, such as, the convex function of Legendre type and the Bregman-Tweedie divergence. Also, we show that the Bregman-Tweedie loss function can be derived from the regular Legendre transformation of the Bregman-Tweedie divergence. The proposed Bregman-Tweedie classification model ($\alpha \in (0,1) \cap \R_e$) have two sub-models; H-Bregman (with hinge-like loss function and $c_{\alpha} = -1$) and L-Bregman (with logistic-like loss function and $c=1$). The H-Bregman and L-Bregman outperform the classic logistic regression, SVM, and L2SVM in terms of the Friedman ranking and show reasonable performance in terms of classification accuracy when $\alpha \approx 1$.

\section*{Acknowledgments}
This paper is supported by the Basic Science Program through the NRF of Korea funded by the Ministry of Education (NRF-2015R101A1A01061261).

\appendix
\section*{Appendix}
In this Appendix, we summarize several useful Lemmas. 

\begin{lemma}\label{lemma4}
There is a one-sided inverse relation between the extended logarithmic function~\eqref{exlog}   and the extended exponential function~\eqref{exexp} within the reduced domain. That is, when $x \in \dom (\ln_{\alpha})$ in Table \ref{tableL}, except the case $x \in \R_{--}$ with $\alpha \in \R_o \setminus \{ 1 \}$, the following is satisfied.
\begin{equation}\label{bijection-1}
\exp_{\alpha}(\ln_{\alpha}(x)) = x.
\end{equation}
In addition, if $y \in \dom (\exp_{\alpha})$ in Table \ref{tableE}, except the case $(1-\alpha)y \in \R_{--}$ with $\alpha \in \R_{xe}$, the following is satisfied.
\begin{equation}\label{bijection-2}
\ln_{\alpha}(\exp_{\alpha}(y))=y  
\end{equation}
\end{lemma}
\begin{proof}
When $\alpha=1$, the extended logarithmic function~\eqref{exlog} and the extended exponential function~\eqref{exexp} become the conventional logarithmic and exponential function. Now, let us assume that $\alpha \not= 1$. From Table \ref{tableE} and Table \ref{tableL}, it is easy to check that $\hbox{ran}(\ln_{\alpha}) \subseteq \dom(\exp_{\alpha})$. Also, we have $\hbox{ran}(\exp_{\alpha}) \subseteq \dom(\ln_{\alpha})$. Therefore, we only need to check one-to-one condition. When $1-\alpha \in \R_o \cup \R_{xx}$, it is easy to check that $\ln_{\alpha}(x) = \frac{1}{1-\alpha}x^{1-\alpha}$ and $\exp_{\alpha}(y) = [(1-\alpha)y]^{1/(1-\alpha)}$ are strictly monotonic function on their domains depending on the choice of $\alpha \in \R \setminus \{ 1 \}$. Hence, one-to-one condition is automatically satisfied. However, when $1-\alpha \in \R_e \cup \R_{xe}$, the condition is rather complicated. Let  $x,y \in \R_+$ with $\alpha<1$ or $x,y \in \R_{--}$ with $1<\alpha$, then it is easy to check one-to-one condition of $\exp_{\alpha}(\ln_{\alpha}(x))=x$ or $\ln_{\alpha}(\exp_{\alpha}(y))=y$. On the other hand, other cases (i.e., $x,y \in \R_{-}$ with $\alpha<1$ or $x,y \in \R_{++}$ with $\alpha>1$) do not satisfy one-to-one condition, due to the inherent square of the exponent in $\R_{e} \cup \R_{xe}$.
\end{proof}

\begin{lemma}\label{psiTh}
Let $\dom(\exp_{\alpha})$ be in Table \ref{table3}. Then $\dom  \Psi$ of 
$\Psi(x) = \int_d^{x} \exp_{\alpha}(\xi)d\xi$ 
is classified below. Here, we drop constant terms.
\begin{itemize}
\item $\alpha=1$:
$
 \Psi(x) = \exp(x)
$\;\;\;\;\; with  $\dom \Psi = \R$
\item $\alpha=2$: 
$
\Psi(x) =  -\log(-x)
$ with $\dom \Psi = \R_{--}$
\item $\alpha \not\in \{1,2\}$:  
$
\Psi(x) = \frac{1}{2-\alpha}[(1-\alpha)x]^{\frac{2-\alpha}{1-\alpha}}. 
$
In this case, $\dom \Psi$ is categorized as
\begin{itemize}
\item $\alpha<1$: \quad\quad $\dom  \Psi = \R$ \hskip 2.6cm if $\alpha \in \R_e$ and $\dom  \Psi = \R_+$ otherwise.
\item $1<\alpha<2$:\; $\dom  \Psi = \R_{++}$ / $\R_{--}$ \hskip 0.5cm if $\alpha \in \R_e$ and $\dom  \Psi = \R_{--}$ otherwise
\item $2<\alpha$: \quad\quad $\dom  \Psi = \R_{-}$ / $\R_+$ \hskip 0.85cm if $\alpha \in \R_e$ and $\dom  \Psi = \R_{-}$ otherwise.
\end{itemize}
\end{itemize}
\end{lemma}
\begin{proof}
By simple calculation, if $1-\alpha \in \R_{xe}$ then $\frac{2-\alpha}{1-\alpha} \in \R_o$ and if $\alpha \in \R_e$ then $\frac{2-\alpha}{1-\alpha} \in \R_e$. As noticed in Lemma \ref{lemma4}, $\exp_{\alpha}$ is monotonically increasing function on its domain in Table \ref{table3}. Therefore, $\Psi(x)$ is a convex function~\citep{hir96} and thus the domain of $\Psi$ need to be defined to satisfy the following equation~\citep{roc70}:
\begin{equation}\label{dompsi}
int(\dom  \Psi) \subseteq \dom \partial \Psi = \dom (\exp_{\alpha}) \subseteq \dom  \Psi,
\end{equation}
where $\dom(\exp_{\alpha})$ is defined as in Table \ref{table3}. Based on \eqref{dompsi} and Table \ref{table3}, we summarize $\dom\Psi$ as follows: 
\begin{itemize}
\item $\alpha<1$: $\frac{2-\alpha}{1-\alpha}>1$ and thus
\begin{itemize}
\item $1-\alpha\in \R_{xe}$: $\frac{2-\alpha}{1-\alpha} \in \R_o$ and thus $\dom \partial\Psi = \R_+$ and $\dom \Psi=\R_+$. 
\item $1-\alpha \in \R_{xx} \cup \R_e$: $\frac{2-\alpha}{1-\alpha} \in \R_x$ and thus $\dom \Psi = \R_+.$
\item $1-\alpha \in \R_o$: $\frac{2-\alpha}{1-\alpha} \in \R_e$ and thus $\dom \Psi = \R$. 
\end{itemize}
\item $1<\alpha<2$: $\frac{2-\alpha}{1-\alpha}<0$ and thus
\begin{itemize}
\item $1-\alpha \in \R_{xe}$:  $\frac{2-\alpha}{1-\alpha} \in \R_o$ and the assumption that the domain should be convex, irrespective of convexity of the function, we need to select one between $\R_{++}$ and $\R_{--}.$ From Table \ref{table3}, we have $\dom  \partial\Psi = \R_{--}$ and thus by \eqref{dompsi}, we need to choose $\dom \Psi = \R_{--}.$
\item $1-\alpha \in \R_{xx}\cup\R_e$: It is natural to restrict $\dom \Psi = \R_{--}$, since $1<\alpha$.
\item $1-\alpha \in \R_o$: $\frac{2-\alpha}{1-\alpha} \in \R_e$ and thus $\dom \Psi = \R_{++}$ or $\R_{--}.$ Both are well matched with $\dom \partial\Psi$ in Table \ref{table3}.
\end{itemize}
\item $2<\alpha$: $1<\alpha$ and $\frac{2-\alpha}{1-\alpha}>0$. Therefore, we have
\begin{itemize}
\item $1-\alpha \in \R_{xe}:$ $\frac{2-\alpha}{1-\alpha} \in \R_o$ and thus we have $\dom \Psi = \R_-.$ 
\item $1-\alpha \in \R_{xx} \cup \R_e:$ $\frac{2-\alpha}{1-\alpha} \in \R_x$ and thus we naturally select $\dom \Psi = \R_{-}.$
\item $1-\alpha \in \R_o:$ Since $\frac{2-\alpha}{1-\alpha} \in \R_e$, we have $\dom \partial\Psi = \R_{--}$ or $\R_{++}.$ Due to \eqref{dompsi}, we have
$int(\dom \Psi) = \R_{--}$ or $\R_{++}.$ Hence, $\dom \Psi = \R_{-}$ or $\R_+.$
\end{itemize}
\end{itemize}
\end{proof}

\end{document}